\documentclass{article}

\usepackage{mathptmx}
\usepackage{soul}\setuldepth{article}

\def\hb{\hbox to 10.7 cm{}}
\usepackage{tikz}
\usepackage[textsize=small, text width=120]{todonotes}

\usepackage{amsmath}
\usepackage{amssymb}
\usepackage{amsthm}

\usepackage{stmaryrd}
\usepackage{mathrsfs}

\usepackage{complexity}
\usepackage{subcaption}

\usepackage{ifthen}
\usepackage{environ}
\usepackage{xspace,url}
\usepackage{hyperref}

\usepgflibrary{patterns}  
\usepgflibrary[patterns]  
\usetikzlibrary{patterns} 
\usetikzlibrary[patterns] 
\newif\iflong
\NewEnviron{shortproof}{\iflong\else\expandafter\begin{proof}\BODY\end{proof}\fi}
\NewEnviron{longproof}{\iflong\expandafter\begin{proof}\BODY\end{proof}\fi}
\longfalse
\newtheorem{theorem}{Theorem}
\newtheorem{lemma}[theorem]{Lemma}
\newtheorem{proposition}[theorem]{Proposition}

\theoremstyle{definition}
\newtheorem{definition}{Definition}
\newtheorem{example}{Example}

\usepackage{tikz}
\usetikzlibrary{shapes,decorations,shadows,matrix,arrows}

\tikzstyle{arg}=[draw, thick, circle, fill=gray!15,inner sep=3pt]

\newcommand{\nop}[1]{}

\newcommand{\parents}[1]{\ensuremath{\mathit{par}(#1)}\xspace}

\newcommand{\ileq}{\ensuremath{\leq_i}\xspace}

\newcommand{\lin}{\ensuremath{\mathtt{in}}\xspace}
\newcommand{\lout}{\ensuremath{\mathtt{out}}\xspace}
\newcommand{\lundec}{\ensuremath{\mathtt{undec}}\xspace}
\newcommand{\Lab}{\mathcal{L}\xspace}
\newcommand{\lab}{\mathbb{L}\xspace}

\newcommand{\update}[3]{#1|^{#2}_{#3}}

\newcommand{\SFSADF}{\text{SFSADF}}

\newcommand{\SETAF}{\text{SETAF}}
\newcommand{\SETADF}{\text{SETADF}}
\newcommand{\SFADF}{\text{SFADF}}

\def\ac{\varphi}

\newcommand{\tvt}{\ensuremath{\mathbf{t}}\xspace}
\newcommand{\tvf}{\ensuremath{\mathbf{f}}\hspace*{0.005mm}\xspace}
\newcommand{\tvu}{\ensuremath{\mathbf{u}}\xspace}

\newcommand{\cf}{\textit{cf}}
\newcommand{\adm}{\textit{adm}}
\newcommand{\prf}{\textit{prf}}
\newcommand{\pref}{\prf}
\newcommand{\com}{\textit{com}}
\newcommand{\comp}{\com}
\newcommand{\grd}{\textit{grd}}
\newcommand{\stb}{\textit{stb}}
\newcommand{\model}{\textit{mod}}
\newcommand{\naive}{\textit{nai}}

\newcommand{\Args}[1]{\ensuremath{\textsc{Args}_{#1}}}

\title{Expressiveness of SETAFs and Support-Free~ADFs under\\ 3-valued Semantics}
		\author{Wolfgang Dvo\v{r}\'{a}k $^a$ \and
		Atefeh Keshavarzi Zafarghandi $^b$ \and
		Stefan Woltran $^a$\medskip\\
		$^a$ Institute of Logic and Computation, TU Wien, Austria\\
		$^b$ Department of Artificial Intelligence, Bernoulli Institute,\\ University of Groningen, The Netherlands
		}
\begin{document}
	  \maketitle
	

		\begin{abstract}
			Generalizing the attack structure in 
			argumentation frameworks (AFs)
			has been studied in different ways. 
			Most prominently, the binary
			attack relation of Dung frameworks has been extended to the 
			notion of collective attacks. The resulting 
			formalism is often termed SETAFs. Another approach is 
			provided via abstract dialectical frameworks (ADFs), where
			acceptance conditions specify the relation between
			arguments; restricting these conditions naturally allows for
			so-called support-free ADFs. The aim of the paper
			is to shed light on the relation between these two
			different approaches. To this end,
			we investigate and compare the expressiveness of SETAFs 
			and support-free ADFs	
			under the lens of 3-valued semantics.
			Our results show that it is only the presence of
			unsatisfiable acceptance conditions in support-free ADFs
			that discriminate the two approaches.
		\end{abstract}
		
	
	\section{Introduction}
	Abstract argumentation frameworks (AFs) as introduced by Dung~\cite{Dung95}
	are a core formalism in formal argumentation.
	A popular line of research investigates extensions  of Dung AFs
	that allow for a richer syntax (see, e.g.~\cite{DBLP:journals/expert/BrewkaPW14}).
	In this work we investigate two generalisations of Dung AFs that allow for a more flexible attack structure (but do not consider support between arguments).
	
	The first formalism we consider are SETAFs as introduced by Nielsen and Parsons~\cite{nielsen2006generalization}. SETAFs extend Dung AFs by allowing for 
	collective attacks such that a set of arguments $B$ attacks another argument $a$ but no proper subset of $B$
	attacks $a$. 
	Argumentation frameworks with collective attacks have received increasing interest in the last years. For instance, 
	semi-stable, stage, ideal, and eager semantics have been adapted to SETAFs in~\cite{DvorakFW19,flouris2019comprehensive}; 
	translations between SETAFs and other abstract argumentation formalisms are studied in \cite{Polberg17};
	\cite{YunVC18} observed that for particular instantiations, SETAFs provide
	a more convenient target formalism than Dung AFs.
	The expressiveness of SETAFs with two-valued semantics has been investigated 
	in~\cite{DvorakFW19} in terms of signatures. Signatures 
	have been introduced in \cite{DunneDLW15} for AFs. In general terms, 
	a signature for a formalism and a semantics captures all possible 
	outcomes that can be obtained by the instances of the formalism under the considered semantics.
	Besides that, signatures are recognized as crucial for operators in dynamics of argumentation (cf.\ \cite{BaumannB19}).
	
	The second formalism we consider are support-free abstract dialectical frameworks (SFADFs), a subclass of 
	abstract dialectical frameworks (ADFs)~\cite{
		BrewkaESWW18} which are known as an advanced abstract formalism for 
	argumentation, that is able to cover several generalizations of AFs~\cite{DBLP:journals/expert/BrewkaPW14,Polberg17}. 
	This is accomplished
	by acceptance conditions which specify, for each argument, its relation
	to its neighbour arguments via propositional formulas. These conditions
	determine the links between the arguments which can be, in particular, 
	attacking or supporting.
	SFADFs are ADFs where each link between
	arguments is attacking; they have been introduced
	in a recent study on 
	different sub-classes of ADFs~\cite{DBLP:journals/argcom/DillerZLW20}.
	
	For comparison of the two formalisms, we need to focus on 
	3-valued (labelling) semantics~\cite{verheij1996two,CaminadaG09}, which are integral for ADF semantics~\cite{BrewkaESWW18}.
	In terms of SETAFs, we can rely on the recently introduced labelling
	semantics in~\cite{flouris2019comprehensive}. We first define a new
	class of ADFs (SETADFs) where the acceptance conditions strictly 
	follow the nature of collective attacks in SETAFs and show 
	that SETAFs and SETADFs 
	coincide for the main semantics, i.e.\
	the $\sigma$-labellings of a SETAF are equal to the $\sigma$-interpretations
	of the corresponding SETADF.
	We then provide 
	exact characterisations of the 3-valued signatures for SETAFs 
	(and thus for SETADFs) for most of the semantics under consideration.
	While
	SETADFs are a syntactically defined subclass of ADFs, the second
	formalism we study
	can be understood as semantical subclass
	of ADFs. In fact, for SFADFs it is not the syntactic structure of acceptance conditions
	that is restricted but their semantic behavior, in the sense that
	all links need to be attacking. The second main contribution of the
	paper is to 
	determine the exact difference in expressiveness between SETADFs and SFADFs.
	
	We briefly discuss related work.
	The expressiveness of SETAFs has first been investigated in~\cite{LinsbichlerPS16} where different sub-classes of ADFs,
	i.e.\ AFs, SETAFs and Bipolar ADFs, are related w.r.t.\ their signatures of 3-valued semantics. Moreover, they provide an 
	algorithm to decide realizability in one of the formalisms under 
	different semantics. However, no explicit characterisations of the signatures are given.
	Recently, P\"{u}hrer~\cite{Puhrer20} presented explicit characterisations of the signatures of general ADFs
	(but not for the sub-classes discussed above).
	In contrast, \cite{DvorakFW19} provides explicit characterisations of the two-valued signatures of SETAFs and 
	shows that SETAFs are more expressive than AFs. 
	In both works all arguments are relevant for the signature, while
	in~\cite{flouris2019comprehensive} it is shown that when allowing to add extra arguments to an AF
	which are not relevant for the signature, i.e.\ the extensions/labellings are projected on common arguments, then SETAFs and AFs are of equivalent expressiveness.
	Other recent work
	\cite{Wallner19} already implicitly showed
	that SFADFs with satisfiable acceptance conditions 
	can be equivalently represented as SETAFs.  This provides a sufficient condition for rewriting an ADF as SETAF and raises the question 
	whether it is also a necessary condition. In fact, we will show that a SFADF has an equivalent SETAF if and only if all acceptance conditions are satisfiable.
	Different sub-classes of ADFs (including SFADFs) have been compared in~\cite{DBLP:journals/argcom/DillerZLW20}, but no exact characterisations of signatures 
	as we provide
	here are given in that work.
	
	\noindent To summarize, the main contributions of our paper are as follows:
	
	\begin{itemize}
		\item We embed SETAFs under 3-valued labeling based semantics~\cite{flouris2019comprehensive} in the more general
		framework of ADFs. That is, we show 3-valued labeling based SETAF semantics 
		to be equivalent to the corresponding ADF semantics. 
		As a side result, this also shows the equivalence of the 3-valued SETAF semantics in \cite{LinsbichlerPS16}  and \cite{flouris2019comprehensive}.
		\item We investigate the expressiveness of SETAFs under 3-valued semantics by providing exact 
		characterizations of the signatures for preferred, stable, grounded and conflict-free semantics, 
		thus complementing the investigations on expressiveness of 
		SETAFs~\cite{DvorakFW19} in terms of extension-based semantics.
		\item We study the relations between SETAFs and support-free ADFs (SFADFs). In particular we give the exact difference in expressiveness between SETAFs and SFADFs
		under conflict-free, admissible, preferred, grounded, complete, stable and two-valued model semantics. 
	\end{itemize}
	
	\noindent Some technical details had to be omitted but are available in an appendix.
	
	\section{Background}
	In this section we briefly recall the necessary definitions for SETAFs and ADFs.
	\begin{definition}\label{def: SETAFS}
		A set argumentation framework (SETAF) is an ordered pair  $F= (A, R)$, where $A$ is a finite set of 
		arguments 
		and $R\subseteq (2^A\setminus\{\emptyset\})\times A$ is the attack relation.  
	\end{definition}
	
	\nop{	
		The semantics of SETAFs can now also be defined similarly to AFs via a characteristic operator. 
		With a slight abuse of notation, we thus define first of all also for a SETAF $F=(A,R)$, 
		
		$\Gamma_F(S)=\{a\in A\ |\ a \text{ is defended by } S \text{ in } F\}$; here the notion of ``defense'' clearly being that defined for SETAFs.  For completeness we detail the definitions of all semantics we consider in this work for SETAFs, although the definitions are exactly as those for AFs (modulo the use of the more general notions of attack and the characteristic operator for SETAFs): 
		
		\begin{definition}\label{def:semanticsSETAF}
			Let $F=(A, R)$ be a SETAF.  A set  $S$ which is conflict-free in $F$ is
			\begin{itemize}
				\item \emph{naive} in $F$ iff S is $\subseteq$-maximal among all conflict-free sets;
				\item  \emph{admissible} in $F$ iff $S \subseteq \Gamma_F(S)$;
				\item  \emph{complete} in $F$ iff $S = \Gamma_F(S)$;
				\item  \emph{grounded} in $F$ iff $S$ is the $\subseteq$-least fixed-point of $\Gamma_F$; 
				\item  \emph{preferred} in $F$ iff $S$ is $\subseteq$-maximal admissible (resp.\ complete) in $F$;
				\item  \emph{stable} in $F$ iff 
				for all $a\in A\setminus S$, $S$ attacks $a$.
			\end{itemize}
	\end{definition}}
	
	The semantics of SETAFs are usually defined similarly to AFs, i.e., based on extensions. However, in this work we focus on 3-valued labelling based semantics, cf.~\cite{flouris2019comprehensive}.  
	
	\begin{definition}
		A (3-valued) labelling of a SETAF $F=(A, R)$ is a total function $\lambda: A \mapsto \{\lin, \lout, \lundec\}$.
		For $x \in \{\lin, \lout, \lundec\}$ we write $\lambda_x$ to denote the sets of arguments $a\in A$ with $\lambda(a)=x$.
		We sometimes denote labellings $\lambda$ as triples $(\lambda_\lin,\lambda_\lout,\lambda_\lundec)$.
	\end{definition}
	
	\begin{definition}\label{def:semanticsSETA_labF}
		Let $F=(A, R)$ be a SETAF. 
		A labelling is called conflict-free in $F$ if 
		(i)~for all $(S,a) \in R$ either $\lambda(a) \not= \lin$ or there is a $b \in S$ with $\lambda(b) \not= \lin$, and
		(ii)~for all $a \in A$, if $\lambda(a) = \lout$ then there is an attack $(S,a) \in R$ such that $\lambda(b) = \lin$ for all $b \in S$.
		A labelling  $\lambda$ which is conflict-free in $F$ is
		\begin{itemize}
			\item  \emph{admissible} in $F$ iff 
			for all $a \in A$ if $\lambda(a) = \lin$ then for all  $(S,a) \in R$ 
			there is a $b \in S$ such that $\lambda(b) = \lout$;
			\item  \emph{complete} in $F$ iff for all $a \in A$ 
			(i)  $\lambda(a) = \lin$ iff for all  $(S,a) \in R$ 
			there is a $b \in S$ such that $\lambda(b) = \lout$, and
			(ii)  $\lambda(a) = \lout$ iff there is an attack  $(S,a) \in R$ 
			such that $\lambda(b) = \lin$ for all $b \in S$;			
			\item  \emph{grounded} in $F$ iff it is complete 
			and 
			there is no $\lambda'$ with $\lambda'_{\lin} \subset \lambda_{\lin}$ complete in $F$;
			\item  \emph{preferred} in $F$ iff it is complete 
			and there is no $\lambda'$ with $\lambda'_{\lin}\! \supset\! \lambda_{\lin}$ complete in $F$;
			\item  \emph{stable} in $F$ iff $\lambda_{\lundec}=\emptyset$.
		\end{itemize}
	\end{definition}
	\noindent
	The set of all $\sigma$ labellings for a SETAF $F$ is denoted by $\sigma_\Lab(F)$, where $\sigma\in\{\cf, \adm$, $\comp, \grd, \pref, \stb\}$ abbreviates the different semantics in the obvious manner. 
	
	\nop{
		
		\textbf{Remark:} Differences to the definitions in \cite{LinsbichlerPS16}: \cite{LinsbichlerPS16} defines admissible, complete, preferred, and mod as 3-valued interpretations (not labellings), i.e.\ grounded, and stable are missing. preferred is defined via admissible.
		They already provide equivalence results for their semantics and the corresponding ADF semantics.
		
		The basic
		definitions of ADFs and semantics of ADFs  are derived from those given in~\cite{
			BrewkaESWW13}.
	}
	
	\begin{example}\label{exp: SETAF}
		The SETAF $F=(\{a,b, c\}, \{(\{a, b\}, c), (\{a,c\}, b)\})$ is depicted in Figure~\ref{fig: SETAF}. For instance,  $(\{a, b\}, c)\in R$ says that  there is a joint attack from $a$ and $b$ to $c$.  
		This represents that neither $a$ nor $b$ is strong enough to attack $c$ by themselves. Further,  $\{a\mapsto\lin, b\mapsto\lundec, c\mapsto\lin\}$ is an instance of a  conflict-free labelling, that is  not an admissible labelling (since $ c$ is mapped to $\lin$ but neither $a$ nor $b$ is mapped to $\lout$). The labelling that maps all argument to $\lundec$ is not a complete labelling, however, it is an admissible labelling. Further, $\{a\mapsto\lin, b\mapsto\lundec, c\mapsto\lundec\}$ is an admissible, the unique grounded and a complete labelling, which is not a preferred labelling because $\lambda_{\lin}=\{a\}$ is not $\subseteq$-maximal among all complete labellings. Moreover, $\prf_\Lab(F)= \stb_\Lab(F)= \{\{a\mapsto\lin, b\mapsto\lout, c\mapsto\lin\},  \{a\mapsto\lin, b\mapsto\lin, c\mapsto\lout\}\}$. 
		
		\begin {figure}[t]
		\centering
		\begin {tikzpicture}[-,>=stealth,node distance=0.6cm,
		thick,main node/.style={circle,fill=none,draw,minimum size = 0.4cm,font=\normalsize\bfseries},
		condition/.style={fill=none,draw=none,font=\small\bfseries},scale=0.5]
		\path
		(0,0)node[main node] (A) {$a$}
		++(-1.5,-2.5) node[main node] (B) {$b$}
		++(2.8,0)node[main node] (C) {$c$};
		
		\path (A) edge  (0.30,-1.30);
		\path (C) edge  (0.30,-1.30);
		\path[->] (0.30,-1.30) edge  (B);
		\path (A) edge  (-0.30,-1.30);
		\path (B) edge  (-0.30,-1.30);
		\path[->] (-0.30,-1.30) edge  (C);		
		
	\end{tikzpicture}
	\caption{The SETAF of Example \ref{exp: SETAF}.}
	\label{fig: SETAF}
\end{figure}
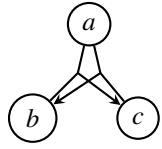  
\end{example}
\nop{
Notice that \emph{Dungs Abstract Argumentation Frameworks (AFs)}~\cite{Dung95} and their semantics 
can be identified with SETAFs whose attacks are restricted a single argument attacking an argument~\cite{nielsen2006generalization}.
That is all attacks are of the form $({b},a)$ for some arguments $a,b$ (in the setting of Dung AFs such attacks are then denoted
by pairs of arguments $(b,a)$).
}

We next turn to abstract dialectical frameworks~\cite{
BrewkaESWW13}.

\begin{definition}
An abstract dialectical framework (ADF) is a tuple $D = (S, L, C)$
where:
\begin{itemize}
	\item $S$ is a finite set of arguments (statements, positions);
	\item $L\subseteq S \times S$ is a set of links among arguments;
	\item $C = \{\varphi_s\}_{s\in S}$ is a collection of propositional formulas over arguments, called acceptance conditions.
\end{itemize}
\end{definition}
\vspace{-0.2cm}
An ADF can be represented by a graph in which nodes indicate arguments and links show the relation among arguments. Each argument $s$ in an ADF is attached by a propositional formula, called acceptance condition,  $\varphi_s$ over $\parents s$ such that, $\parents s=\{ b\ |\ (b, s)\in L\}$.  
Since in ADFs an argument appears in the acceptance condition of an argument $s$ if and only if it belongs to the set $\parents{s}$, the set of links $L$ of an ADF is given implicitly via the acceptance conditions. 
The acceptance condition of each argument clarifies under which condition the argument can be accepted and determines the type of links   
(see Definition~\ref{def: type of links} below).
An \emph{interpretation} $v$ (for $F$) is a function $v : S \mapsto \{\tvt, \tvf , \tvu\}$, that maps
arguments to one of the three truth values
true ($\tvt$), false ($\tvf$), or undecided ($\tvu$). Truth values can be ordered via information ordering relation $<_i$ given by  $\tvu <_i \tvt$ and $\tvu <_i \tvf$ and no other pair of truth values are related by $<_i$. Relation $\leq_i$ is the reflexive and transitive closure of $<_i$.
An interpretation $v$ is \emph{two-valued} if it maps each argument to either $\tvt$ or $\tvf$.
Let $\mathcal{V}$ be the set of all interpretations for an ADF $D$. Then, we call a subset of all interpretations of the ADF, $\mathbb{V}\subseteq \mathcal{V}$, an \textit{interpretation-set}.
Interpretations can be ordered via $\leq_i$ with respect to  their information content,
i.e.\ $w \leq_i v$ if $w(s) \leq_i v(s)$ for each $s \in S$.
\nop{It is said that an interpretation $v$  is an \textit{extension}\ of another interpretation $w$, if $w(s) \leq_i v(s)$ for each $s \in S$, denoted by $w \leq_i v$. Interpretations $v$ and $w$ are incomparable if neither $w\nleq_i v$ nor $v \nleq_i w$, denoted by $w \not\sim v$.}  Further,  we denote the update of an interpretation $v$ with a truth value $x\in\{\tvt, \tvf, \tvu\}$ for an argument $b$ by
${v}|^{b}_{x}$,
i.e.\ ${v}|^{b}_{x}(b) = x$ and ${v}|^{b}_{x}(a) = v(a)$ for $a\neq b$.
Finally, the partial valuation of acceptance condition $\varphi_s$ by $v$, is given by
$\varphi_s^v = v(\ac_s) = \varphi_s[p/\top : v(p)=\tvt][p/\bot : v(p)=\tvf]$,  for $p\in\parents{s}$.

Semantics for ADFs can be defined via a
\emph{characteristic operator} $\Gamma_D$ for an ADF $D$. 
Given an interpretation $v$ (for $D$), the characteristic operator $\Gamma_D$ for $D$ is defined as
\[
\Gamma_D(v) = v' \text{ such that } v'(s)=\begin{cases}
\tvt &\quad\text{if $\varphi_s^v$} \text{ is irrefutable (i.e.,  a tautology)},\\
\tvf &\quad\text{if $\varphi_s^v$}\text{ is unsatisfiable}, \\
\tvu &\quad\text{otherwise.}
\end{cases}
\]

\begin{definition}
\label{def:sem:adfs}
Given an ADF $D=(S, L, C)$, an interpretation $v$ is
\begin{itemize}
	\item \emph{conflict-free} in $D$ iff $v(s)= \tvt$ implies $\varphi_s^v$ is satisfiable and $v(s)=\tvf$ implies $\varphi_s^v$ is unsatisfiable;
	\item  \emph{admissible} in $D$ iff $v \ileq \Gamma_D(v)$;
	\item  \emph{complete} in $D$ iff $v = \Gamma_D(v)$;
	\item  \emph{grounded} in $D$ iff $v$ is the least fixed-point of $\Gamma_D$; 
	\item  \emph{preferred} in $D$ iff $v$ is $\ileq$-maximal admissible  in $D$;
	\item  a \emph{(two-valued) model} of $D$ iff $v$ is two-valued and for all $s \in S$, it holds that $v(s)=v(\varphi_s)$;
	\item  a \emph{stable model} of $D$ if $v$ is a model of $D$ and $v^\tvt=w^\tvt$, where $w$ is
	the grounded interpretation of the $\stb$-reduct $D^v = (S^v, L^v, C^v)$, where
	$S^v=v^\tvt$,
	$L^v=L\cap(S^v \times S^v)$, and
	$\ac_s[p/\bot : v(p)=\tvf]$ for each $s \in S^v$.
\end{itemize}
The set of all $\sigma$ interpretations for an ADF $D$ is denoted by $\sigma(D)$, where $\sigma\in\{\cf, \adm$, $\comp, \grd, \pref,\model, \stb\}$ abbreviates the different semantics in the obvious manner. 
\end{definition}
\nop{

Intuitively, the idea of defining stable models of ADFs follows the idea of stable models of logic programming, that breaks self-justify support cycles.
In fact, in ADF $D$, a model $v$ is a stable model if there exists a constructive proof for all arguments assigned to true in $v$, if all arguments which are assigned to false in $v$ are actually false. 
Since in AFs and SETAFs there is no direct support link, stable models and models are equal. The relation among semantics of  ADF $D$ are as follows: $\stb(D)\subseteq\model(D)\subseteq\prf(D)\subseteq\com(D)\subseteq\adm(D)\subseteq\cf(D)$, further, $\grd(D)\subseteq\com(D)$. The same as AF each ADF contains at least one admissible, preferred, complete, and grounded interpretation, however the existence of stable models, and models respectively, is not guaranteed. 

}

\begin{example}\label{exp: ADF}
An example of an ADF $D=(S,L,C)$ is shown in Figure \ref{fig: ADF}. To each argument a propositional formula is associated, the acceptance condition of the argument.  
For instance, the acceptance condition of $c$, namely $\varphi_c: \neg a\lor \neg b$, states that $c$ can be accepted in an interpretation where either $a$ or $b$ (or both) are rejected. 

In $D$ the interpretation $v=\{a\mapsto\tvu, b\mapsto \tvu, c\mapsto\tvt\}$ is conflict-free.  
However, $v$ is not an admissible interpretation, because $\Gamma_D(v)=\{a\mapsto\tvu, b\mapsto\tvu, c\mapsto\tvu\}$, that is, $v\not\leq_i\Gamma_D(v)$. The interpretation $v_1=\{a\mapsto\tvf, b\mapsto\tvt, c\mapsto\tvu\}$ on the other hand is an admissible interpretation. Since $\Gamma_D(v_1)=\{a\mapsto\tvf, b\mapsto\tvt, c\mapsto\tvt\}$ and $v_1\leq_i\Gamma_D(v_1)$.  Further, $\prf(D) =\model(D)=\{ \{a\mapsto\tvt, b\mapsto\tvf, c\mapsto\tvt\}, \{a\mapsto\tvf, b\mapsto\tvt, c\mapsto\tvt\}\}$, but only the first interpretation in this set is a stable model. This is because for $v=\{a\mapsto\tvt, b\mapsto\tvf, c\mapsto\tvt\}$ the unique grounded interpretation $w$ of $D^v$ is $\{a\mapsto\tvt, c\mapsto\tvt\}$ and $v^\tvt= w^\tvt$.  The interpretation  $v'= \{a\mapsto\tvf, b\mapsto\tvt, c\mapsto\tvt\}$ is not a stable model, since the unique grounded interpretation $w'$ of $D^{v'}$ is $\{b\mapsto \tvu, c\mapsto\tvt\}$ and $v'^\tvt\not=w'^\tvt$. Actually, $v'$ is not a stable model because the truth value of $b$ in $v'$ is since of self-support. Moreover, the unique grounded interpretation of $D$ is $v=\{a\mapsto\tvu, b\mapsto\tvu, c\mapsto\tvu\}$.   
In addition, we have 
$\com(D)=\prf(D)\cup\grd(D)$.

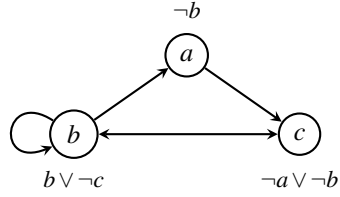
\begin {figure}[t]
\centering
\begin {tikzpicture}[->,>=stealth,node distance=0.6cm,
thick,main node/.style={circle,fill=none,draw,minimum size = 0.4cm,font=\normalsize\bfseries},
condition/.style={fill=none,draw=none,font=\small\bfseries}, yscale=0.7]
\path
(0,0)node[main node] (B) {$b$}
++(1.5,1.5) node[main node] (A) {$a$}
++(1.5,-1.5)node[main node] (C) {$c$};

\path (B) edge  (A);
\path (A) edge  (C);
\path[<->] (B) edge  (C);
\path [loop left,thick, distance=0.9cm, in =210, out=140, ->] (B) edge (B);

\node[condition](ca) [above of= A] {$\neg b$};
\node[condition](cb) [below of= B] {$b\lor \neg c$};
\node[condition](cc) [below of= C] {$ 
	\neg a\lor \neg b$};

\end{tikzpicture}
\caption{The ADF of Example \ref{exp: ADF}.}
\label{fig: ADF}
\end{figure}
\end{example}

In ADFs links between arguments  can be classified into four types, reflecting the relationship of attack and/or support that exists among the arguments.  
In Definition~\ref{def: type of links} we consider two-valued interpretations that are only defined over the parents of $a$, that is,   only give  values to $\parents{a}$.
\begin{definition}\label{def: type of links}
Let $D=(S,L,C)$ be an ADF.
A link $(b,a)\in L$ is called 
\begin{itemize}
\item \emph{supporting} (in $D$)
if for every two-valued interpretation $v$ of $\parents{a}$, $v(\ac_a)=\tvt$ implies $\update{v}{b}{\tvt}(\ac_a)=\tvt$;
\item \emph{attacking} (in $D$)
if for every two-valued interpretation $v$ of $\parents{a}$, $v(\ac_a)=\tvf$ implies $\update{v}{b}{\tvt}(\ac_a)=\tvf$;
\item \emph{redundant} (in $D$) if it is both attacking and supporting;
\item \emph{dependent} (in $D$) if it is neither attacking nor supporting.
\end{itemize} 
\end{definition}

The classification of the types of the links of ADFs is also relevant for classifying ADFs themselves.  One particularly important subclass of ADFs is that of \emph{bipolar} ADFs or BADFs for short. 
In such an ADF each link is either attacking or supporting (or both; thus, the links can also be redundant). 
Another subclass of ADFs, having only attacking links, is defined in~\cite{Atefeh},  called \textit{support free ADFs} (SFADFs) in the current work, defined formally 
as follows.
\begin{definition}\label{def: SFADFs}
An ADF 
is called support-free 
if it has only attacking links.
\end{definition}

For SFADFs, it turns out that the intention of stable semantics, i.e.\ to avoid cyclic support among arguments, becomes immaterial, thus $\model(D)= \stb(D)$ for any ADF $D$; 
the property is called weakly coherent in
\cite{Atefeh}. 

\begin{proposition}\label{prop:mod=stb}
For every SFADF $D$ it holds that   $\model(D)=\stb(D)$.
\end{proposition}
\begin{proof}
The result follows from the following observation:
Let $D=(S, L, C)$ be an ADF, let $v$ be a model of $D$ and let $s\in S$ be an argument such that all parents of $s$ are attackers. 
Thus,  $\varphi_{s}^v$ is irrefutable if and only if $\varphi_{s}[p/\bot\ : \ v(p)=\tvf]$ is irrefutable.  
\end{proof}

\section{Embedding SETAFs in ADFs}
As observed by Polberg~\cite{Polberg16} and Linsbichler et.al~\cite{LinsbichlerPS16},
the notion of collective attacks can also be represented in ADFs by using the right acceptance conditions. 
We next introduce the class SETADFs of ADFs for this purpose.
\begin{definition}\label{def: SETADF}
An ADF $D=(S, L, C)$ is called SETAF-like (SETADF) if 
each of the acceptance conditions in $C$ is given by a formula
(with $\mathcal{C}$ a set of non-empty clauses)
\[
\bigwedge_{cl \in \mathcal{C}}\ \bigvee_{a\in cl}\neg a.
\]
\end{definition}
That is, in a SETADF each acceptance condition is either $\top$ (if $\mathcal{C}$ is empty) 
or a proper CNF formula over negative literals.
SETADFs and SETAFs can be embedded in each other as follows.

\begin{definition}\label{def: SETAFS to SADFs}
Let $F=(A, R)$ be a SETAF. The ADF associated to $F$ is a tuple $D_F=(S, L, C)$ in which $S=A$, $L=\{(a,b)\mid (B,b) \in R, a \in B\}$ and  $C=\{\varphi_{a}\}_{a  \in S}$ is the collection of acceptance conditions defined, for each $a \in S$, as 
\[
\varphi_{a}=
\bigwedge_{(B, a)\in R}\bigvee_{a'\in B}\neg a'.
\] 

Let $D=(S, L, C)$ be a SETADF. We construct the SETAF $F_D=(A,R)$ in which, $A=S$, and 
$R$ is constructed as follows.
For each argument $s \in S$ with acceptance formula $\bigwedge_{cl \in \mathcal{C}}\ \bigvee_{a\in cl}\neg a$
we add the attacks $\{(cl,s) \mid cl \in \mathcal{C}\}$ to $R$.
\end{definition}
Clearly the ADF $D_F$ associated to a SETAF $F$ is a SETADF
and $D$ is the ADF associated to the constructed SETAF $F_D$.
We next deal with the fact that SETAF semantics are defined as three-valued labellings 
while semantics for ADFs are defined as three valued interpretations. 
In order to compare these semantics we associate the $in$ label with $t$, the $out$ label with $f$, and the $undec$ label with $u$.

\begin{theorem}\label{thm:setaf_setadf}
For $\sigma \in \{\cf,\adm,\comp,\pref,\grd,\stb\}$, 
a SETAF~$F$ and its associated SET\-ADF~$D$, we have that 
$\sigma_\Lab(F)$ and $\sigma(D)$ are in one-to-one correspondence with 
each labelling $\lab \in \sigma_\Lab(F)$ corresponding to an interpretation $v \in \sigma(D)$ 
such that
$v(s)=\tvt$ iff $\lambda(s)=\lin$,
$v(s)=\tvf$ iff $\lambda(s)=\lout$, and 
$v(s)=\tvu$ iff $\lambda(s)=\lundec$.
\end{theorem}

Notice that by the above theorem we have that the 3-valued SETAF semantics introduced in~\cite{LinsbichlerPS16} coincide with the 3-valued labelling based SETAF semantics of~\cite{flouris2019comprehensive} and the model semantics of~\cite{LinsbichlerPS16} corresponds to the stable semantics of~\cite{flouris2019comprehensive}. 

\section{3-valued Signatures of SETAFs}
We adapt the concept of signatures~\cite{DunneDLW15} towards our needs first.

\begin{definition}\label{def:sig}
The signature of SETAFs under a labelling-based semantics $\sigma_\Lab$
is defined as 
$\Sigma_{SETAF}^{{\sigma_\Lab}}=\{ \sigma_\Lab(F) \mid F \in SETAF\}$.
The signature of an ADF-subclass ${\cal C}$ under a semantics $\sigma$ is defined as
$\Sigma_{{\cal C}}^{{\sigma}}=\{ \sigma(D) \mid D \in {\cal C}\}$.
\end{definition}

By Theorem~\ref{thm:setaf_setadf} we can use labellings of SETAFs and interpretations of the SETADF class of ADFs interchangeably, yielding that $\Sigma_{SETAF}^{\sigma_\Lab}\equiv \Sigma_{SETADF}^\sigma$, i.e.\
the 3-valued signatures of SETAFs and SETADFs only differ in the naming of the labels.
For convenience, we will use the SETAF terminology in this section.

\begin{proposition}\label{prop:sig_stb} The signature $\Sigma_{SETAF}^{\stb_\Lab}$ is given by all sets $\lab$
of labellings such that 
\begin{enumerate} 
\item all $\lambda \in \lab$ have the same domain 
$\Args{\lab}$;
$\lambda(s) \not= \lundec$ for all $\lambda \in \lab$, $s\in\Args{\lab}$. 
\item If $\lambda \in \lab$ assigns one argument to $\lout$ then it also assigns an argument to $\lin$.
\item For arbitrary $\lambda_1, \lambda_2 \in \lab$ with $\lambda_1 \not= \lambda_2$ 
there is an argument $a$ such that $\lambda_1(a)=\lin$ and $\lambda_2(a)=\lout$.
\end{enumerate}	
\end{proposition}
\begin{proof}
We first show that for each SETAF $F$ the set $\stb_\Lab(F)$ satisfies the conditions of the proposition.
First clearly all $\lambda \in \stb_\Lab(F)$ have the same domain and by the definition of stable 
semantics do not assign $\lundec$ to any argument. That is the first condition is satisfied.
For Condition (2), towards a contradiction assume that the domain is non-empty and $\lambda \in \stb_\Lab(F)$ assigns all arguments to $\lout$.
Consider an arbitrary argument $a$. By definition of stable semantics $a$ is only labeled 
$\lout$ if there is an attack $(B,a)$ such that all arguments in $B$ are labeled in $\lin$,
a contradiction.	 
Thus 
we obtain that there is at least one argument $a$ with $\lambda(a)=\lin$.
For Condition (3), towards a contradiction assume that for all arguments $a$ with $\lambda_1(a)=\lin$ also 
$\lambda_2(a)=\lin$ holds. As $\lambda_1 \not = \lambda_2$  there is an $a$ with 
$\lambda_2(a)=\lin$ and $\lambda_1(a)=\lout$. 
That is, there is an attack $(B,a)$ such that 
$\lambda_1(b)=\lin$ for all $b \in B$. But then also
$\lambda_2(b)=\lin$ for all $b \in B$ and by $\lambda_2(a)=\lin$ we obtain 
that $\lambda_2 \not\in \cf_\Lab(F)$, a contradiction.

Now assume that $\lab$ satisfies all the conditions. We give a SETAF $F_\lab=(A_\lab,R_\lab)$ with  $A_{\lab} = \Args{\lab}$ and $R_\lab   = \{(\lambda_{\lin}, a) \mid \lambda \in \lab, \lambda(a)=\lout\} $.
We show that $\stb_\Lab(F_\lab) = \lab$. 

\nop{
with $\stb_\Lab(F_\lab) = \lab$.
\begin{align*}
A_{\lab} &= \Args{\lab}\\
R_\lab   &= \{(\lambda_{\lin}, a) \mid \lambda \in \lab, \lambda(a)=\lout\}  
\end{align*}}
To this end 
we first show $\stb_\Lab(F_\lab) \supseteq \lab$.
Consider an arbitrary $\lambda \in \lab$: By Condition~(1) there is no $a \in \Args{\mathbb{L}}$
with $\lambda(a)=\lundec$ and it only remains to show $\lambda \in \cf_\Lab(F_\lab)$.
First, if $\lambda(a)=\lout$ for some argument
$a$ then by construction of $R_\lab$ and Condition (2) we have an attack $(\lambda_{\lin}, a)$ 
and thus $a$ is legally labeled $\lout$.
Now towards a contradiction assume there is a conflict $(B,a)$ such that $B \cup \{a\} \subseteq \lambda_{\lin}$.  
Then, by construction of $R_\lab$ there is a $\lambda' \in \lab$  with $\lambda'_{\lin}=B$ 
and $\lambda_{\lin}\not=B$ (as $a \in \lambda_{\lin}$).
That is, $\lambda'_{\lin} \subset \lambda_{\lin}$, a contradiction to Condition~(3).
Thus, $\lambda \in \cf_\Lab(F_\lab)$ and therefore $\lambda \in \stb_\Lab(F_\lab)$.

To show $\stb_\Lab(F_\lab) \subseteq \lab$, 
consider $\lambda \in \stb_\Lab(F_\lab)$. 
If $\lambda$ maps all arguments to $\lin$ then there is no attack in $R_\lab$ which means that
$\lab$ contains only the labelling $\lambda$.
Thus, we  assume that there is $a$ with $\lambda(a)=\lout$  and there is $(B,a) \in R_\lab$
with $B\subseteq\lambda_{\lin}$. 
By construction there is $\lambda' \in \lab$ such that $\lambda'_\lin=B$.
Then by construction we have $(B,c) \in R_\lab$ for all $c \not\in B$ and thus 
$\lambda'_\lin = B = \lambda_\lin$ and moreover $\lambda'_\lout = \lambda_\lout$ and 
thus $\lambda=\lambda'$.
\end{proof}

We now turn to the signature for preferred semantics. Compared to the conditions
for stable semantics, labelling may now assign $\lundec$ to arguments. 
Note that stable is the only semantics allowing for an empty labelling set.

\begin{proposition}\label{prop:sig_pref} The signature $\Sigma_{SETAF}^{\pref_\Lab}$ is given by all non-empty sets $\lab$ of labellings s.t. 
\begin{enumerate}
\item all labellings $\lambda \in \lab$ have the same domain
$\Args{\lab}$.
\item If $\lambda \in \lab$ assigns one argument to $\lout$ then it also assigns an argument to $\lin$.
\item For arbitrary $\lambda_1, \lambda_2 \in \lab$ with $\lambda_1 \not= \lambda_2$ 
there is an argument $a$ such $\lambda_1(a)=\lin$ and $\lambda_2(a)=\lout$.
\end{enumerate}	
\end{proposition}
\begin{proof}[Proof sketch]
We first show that for each SETAF $F$ the set $\pref_\Lab(F)$ satisfies the conditions of the proposition.
The first condition is satisfied as all $\lambda \in \pref_\Lab(F)$ have the same domain. 
The second condition is satisfied by the definition of conflict-free labellings. 
Condition~(3) is by the $\subseteq$-maximality of $\lambda_\lin$ which implies that there is a conflict between each two preferred extensions.
\nop{
Now, assume that $\lambda \in \pref_\Lab(F)$ assigns an argument $a$ to $\lout$.
By the definition of conflict-free labellings there is an attack $(B,a)$ such that all arguments $b\in B$ are labeled $\lin$. Thus condition 2 is satisfied.	 
For condition 3, towards a contradiction assume that for all arguments $a$ with $\lambda_1(a)=\lin$ also 
$\lambda_2(a)=\lin$ holds. 
As for preferred semantics the $\lout$ labels are fully determined by the $\lin$ labels
and $\lambda_1 \not = \lambda_2$  there is an $a$ with 
$\lambda_2(a)=\lin$ and $\lambda_1(a)\not= \lin$. 
This is in contradiction to the $\subseteq$-maximality of $\lambda_\lin$.}

Now assume that $\lab$ satisfies all the conditions. We give a SETAF $F_\lab=(A_\lab,R_\lab)$ with $A_{\lab} = \Args{\lab}$ and $R_\lab   = \{(\lambda_{\lin}, a) \mid \lambda \in \lab, \lambda(a)=\lout\}  \cup 
\{ (\lambda_{\lin} \cup \{a\}, a) \mid \lambda \in \lab, \lambda(a)=\lundec\}$. 
\nop{
with $\pref_\Lab(F_\lab) = \lab$.
\begin{align*}
A_{\lab} &= \Args{\lab}\\
R_\lab   &= \{(\lambda_{\lin}, a) \mid \lambda \in \lab, \lambda(a)=\lout\}  \cup 
\{ (\lambda_{\lin} \cup \{a\}, a) \mid \lambda \in \lab, \lambda(a)=\lundec\}
\end{align*}}
It remains to show that $\pref_\Lab(F_\lab) = \lab$. To show $\pref_\Lab(F_\lab) \supseteq \lab$,  consider an arbitrary $\lambda \in \lab$. $\lambda \in \cf_\Lab(F_\lab)$ can be seen by construction, and $\lambda \in \adm_\Lab(F_\lab)$ since argument labelled out is attacked by $\lambda$; finally  $\lambda \in \prf_\Lab(F_\lab)$ is guaranteed since the arguments $a$ with 
$\lambda(a)=\lundec$ are involved in self-attacks.
To  show $\pref_\Lab(F_\lab) \subseteq \lab$
consider $\lambda \in \pref_\Lab(F_\lab)$. It can be checked that $\lambda$ satisfies all the conditions of the proposition. 
\end{proof}

\begin{proposition}\label{prop:sig_cf}\label{prop: cf} The signature $\Sigma_{SETAF}^{\cf_\Lab}$ is given by all non-empty sets $\lab$
of labellings s.t.
\begin{enumerate}
\item all $\lambda \in \lab$ have the same domain $\Args{\lab}$.
\item If $\lambda \in \lab$ assigns one argument to $\lout$ then it also assigns an argument to $\lin$.
\item For $\lambda \in \lab$ and $C \subseteq \lambda_\lin$ 
also $(C,\emptyset, \Args{\lab} \setminus C) \in \lab$.
\item For $\lambda \in \lab$ and $C \subseteq \lambda_\lout$ 
also $(\lambda_\lin, \lambda_\lout \setminus C, \lambda_\lundec \cup C) \in \lab$.
\item For $\lambda, \lambda' \in \lab$ with $\lambda_\lin \subseteq \lambda'_\lin$ 
also $(\lambda'_\lin,\lambda_\lout \cup \lambda'_\lout, \lambda_\lundec \cap \lambda'_\lundec) \in \lab$.
\item For $\lambda, \lambda' \in \lab$ and $C \subseteq \lambda_\lout$ (s.t. $C\not =\emptyset$)
we have $\lambda_\lin \cup C \not\subseteq \lambda'_\lin$.
\end{enumerate}	
\end{proposition}

\begin{proof}[Proof sketch]
Let $F$ be an arbitrary SETAF we show that $\cf_\Lab(F)$ satisfies the conditions of the proposition. The first two conditions are clearly satisfied by the definition of conflict-free labelling. 
For Condition (3),  towards a contradiction assume that $(C,\emptyset, \Args{\lab} \setminus C)$
is not conflict-free. Then there is an attack $(B,a)$ such that $B \cup \{a\} \subseteq C\subseteq \lambda_{\lin}$,  and thus $\lambda \not\in \cf_\Lab(F)$,
a contradiction. Condition (4) is satisfied as in the definition of conflict-free labellings 
there are no conditions for labeling an argument $\lundec$. Further, the conditions that allow to label an argument $\lout$ solely depend on the $\lin$ labeled arguments. For Condition (5), consider $\lambda, \lambda' \in \cf_\Lab(F)$ with $\lambda_\lin \subseteq \lambda'_\lin$ 
and $\lambda^*=(\lambda'_\lin,\lambda_\lout \cup \lambda'_\lout, \lambda_\lundec \cap \lambda'_\lundec)$. Since $\lambda, \lambda'\in \lab$, it is easy to check that $\lambda^*$ is a well-founded labelling and $\lambda^* \in \cf_\Lab(F)$. For Condition (6), consider $\lambda, \lambda' \in \cf_\Lab(F)$ and
a set $C\subseteq\lambda_{\lout}$ containing an argument $a$ such that $\lambda(a)=\lout$.
That is, there is an attack $(B,a)$ with $B \subseteq \lambda_\lin$
and thus $\lambda_\lin \cup C \not\subseteq \lambda'_\lin$. That is, Condition (6) is satisfied.

Now assume that $\lab$ satisfies all the conditions. We give a SETAF $F_\lab=(A_\lab,R_\lab)$ with $	A_{\lab} = \Args{\lab}$ and $	R_\lab   = \{(\lambda_{\lin}, a) \mid \lambda \in \lab, \lambda(a)=\lout\}  \cup 
\{ (B, b) \mid b\in B, \nexists \lambda\in \lab: \lambda_\lin=B \}$.
To complete the proof it remains to show that $\cf_\Lab(F_\lab) =\lab$.
\end{proof}

Finally, we give an exact characterisation of the signature of grounded semantics.

\begin{proposition}\label{prop:sig_grd}
The signature $\Sigma_{\SETAF}^{\grd_\Lab}$ is given by sets $\lab$ of labellings such that 
$|\lab|=1$, and
if $\lambda\in \lab$ assigns one argument to $\lout$ then $\lambda_{\lin}\not=\emptyset$.
\end{proposition}

Notice that Proposition~\ref{prop:sig_grd} basically exploits that grounded semantics is a unique status semantics 
based on admissibility. The result thus immediately extends to other semantics satisfying these two properties, e.g.\ to ideal or eager semantics~\cite{flouris2019comprehensive}.

So far, we have provided characterisations for the signatures
$\Sigma_{\SETAF}^{\stb_\Lab}$,
$\Sigma_{\SETAF}^{\pref_\Lab}$,
$\Sigma_{\SETAF}^{\cf_\Lab}$,
$\Sigma_{\SETAF}^{\grd_\Lab}$.
By Theorem~\ref{thm:setaf_setadf} we get analogous characterizations of $\Sigma_{\SETADF}^{\sigma}$ for the corresponding ADF semantics.

We have not yet touched admissible and complete semantics. Here, the exact
characterisations seem to be more cumbersome and are left for future work. 
However, for admissible semantics 
the following proposition provides necessary conditions for an labelling-set to be $\adm$-realizable,
but it remains open whether they are also sufficient.
\begin{proposition}\label{prop:sig_adm} For each $\lab \in \Sigma_{SETAF}^{\adm_\Lab}$ we have:
\begin{enumerate}
\item all $\lambda \in \lab$ have the same domain
$\Args{\lab}$.
\item If $\lambda \in \lab$ assigns one argument to $\lout$ then it also assigns an argument to $\lin$.
\item For $\lambda, \lambda' \in \lab$ and $C \subseteq \lambda_\lout$ (s.t. $C\not=\emptyset$) we have $\lambda_\lin \cup C \not\subseteq \lambda'_\lin$.
\item For arbitrary $\lambda, \lambda' \in \lab$ either
(a) $(\lambda_\lin \cup \lambda'_\lin, \lambda_\lout \cup \lambda'_\lout, \lambda_\lundec \cap \lambda'_\lundec) \in \lab$ or 
(b) there is an argument $a$ such $\lambda(a)=\lin$ and $\lambda'(a)=\lout$.
\item

For $\lambda, \lambda'\! \in\! \lab$ 
with $\lambda_\lout \subseteq \lambda'_\lout$, 
and $C \subseteq \lambda_\lin \setminus \bigcup_{\lambda^* \in \lab:\ \lambda^*_{\lin}=\lambda'_{\lin}} \lambda^*_{\lout}$ 
we have $(\lambda'_\lin \cup C,\lambda'_\lout, \lambda'_\lundec \setminus C) \in \lab$.

\item For $\lambda, \lambda' \in \lab$ 
with $\lambda_\lin \subseteq \lambda'_\lin$, 
and $C \subseteq \lambda_\lout$ 
we have $(\lambda'_\lin,\lambda'_\lout  \cup C, \lambda'_\lundec \setminus C) \in \lab$.
\item For $\lambda, \lambda' \in \lab$ 
with $\lambda_\lin \subseteq \lambda'_\lin$ and
$\lambda_\lout \supseteq \lambda'_\lout$
we have $(\lambda_\lin,\lambda'_\lout, \Args{\lab} \setminus (\lambda_\lin \cup \lambda'_\lout)) \in \lab$.    
\item $(\emptyset, \emptyset, \Args{\lab}) \in \lab$.
\end{enumerate}	
\end{proposition}
\begin{proof}
We show that for each SETAF $F$ the set $\adm_\Lab(F)$ satisfies the conditions of the proposition.
Conditions (1)--(3) are by the fact that $\adm_\Lab(F) \subseteq \cf_\Lab(F)$.
For Condition~(4), let $\lambda, \lambda'\in \adm_\Lab(F)$ with  $\lambda_\lin\cap\lambda'_{out}= \{\}$ (since each admissible labelling defends itself, $\lambda'_\lin\cap\lambda_{out}= \{\}$). 
Thus, $\lambda^*=(\lambda_\lin \cup \lambda'_\lin, \lambda_\lout \cup \lambda'_\lout, \lambda_\lundec \cap \lambda'_\lundec) $ is a well-defined labelling. Further, since $\lambda, \lambda'\in \adm_\Lab(F)$ it is easy to check that $\lambda^* \in\adm_\Lab(F)$. 
For Condition (5), let $\lambda^* = (\lambda'_\lin \cup C,\lambda'_\lout, \lambda'_\lundec \setminus C) $.   First, $\lambda^*$ is a well-defined labelling. 
Notice that the set $C$ contains arguments defended by $\lambda$ and not attacked by $\lambda'_\lin$.
Now, it is easy to check that $\lambda^*$ meets the condition for being an admissible labelling.
For Condition (6), let $\lambda^* = (\lambda'_\lin,\lambda'_\lout  \cup C, \lambda'_\lundec \setminus C) $. 
Notice that the set $C$ contains only arguments attacked by $\lambda_\lin$ and thus are also attacked by
$\lambda'_\lin$. Thus, starting from the admissible labelling $\lambda'$ we can relabel arguments in $C$ to $\lout$ and obtain that $\lambda^*$ is also an admissible labelling.
For Condition (7), let $\lambda^* = (\lambda_\lin,\lambda'_\lout, \Args{\lab} \setminus (\lambda_\lin \cup \lambda'_\lout)) $. First, $\lambda^* $ is a well-defined labelling. 
We have that setting $\lambda'_\lout$ to $\lout$ is sufficient to make all the $\lin$ labels for arguments in $\lambda'_\lin$ valid and thus are also sufficient to make 
the $\lin$ labels for arguments $\lambda_\lin \subseteq \lambda'_\lin$ valid.
Moreover, as $\lambda_\lout \supseteq \lambda'_\lout$ also labelling arguments $\lambda_\lin$ with $\lin$ is sufficient to make the $\lout$ labels for $\lambda'_\lout$ valid.
Hence, $\lambda^*$ is admissible.
For Condition (8), 
the conditions of admissible labelling for arguments labelled $\lin$ or $\lout$ in 
$(\emptyset, \emptyset, \Args{\lab})$
are clearly met, since there are no such arguments. 
\end{proof}

\section{On the Relation between SETAFs and Support-Free ADFs}\label{sec: expressiveness of SFADF}
In order to compare SETAFs with SFADFs, we can rely on SETADFs (recall Theorem~\ref{thm:setaf_setadf}). In particular, we will compare the signatures $\Sigma_{SETADF}^{{\sigma}}$ and
$\Sigma_{SFADF}^{{\sigma}}$, cf.\ Definition~\ref{def:sig}.
We start with the observation that each SETADF can be rewritten as an equivalent SETADF  that is also a SFADF.\footnote{
As discussed in~\cite{Polberg17}, in general, SETAFs translate to bipolar ADFs
that contain attacking and redundant links. 
However, when we first remove redundant attacks from the SETAF we obtain a SFADF.}

\begin{lemma}\label{lem: SETADF is a SFADF}
For each SETADF  $D=(S, L, C)$  there is an equivalent SETADF $D'=(S, L', C')$
that is also a SFADF, i.e.\
for each $s \in S$, $\varphi_s \in C$, $\varphi'_s \in C'$ we have 
$\varphi_s \equiv \varphi'_s$. 
\end{lemma}
\begin{proof}
Given a SETADF $D$, by Definition~\ref{def: SETADF}, each acceptance condition is a CNF over 
negative literals and thus does not have any support link which is not redundant. 
We can thus obtain $L'$ by removing the redundant links from $L$
and $C'$ by, in each acceptance condition, deleting the clauses that are super-sets of other clauses.
\end{proof}	

By the above we have that $\Sigma_{\SETADF}^{\sigma}\subseteq \Sigma_{\SFADF}^{\sigma}$. 
Now consider the interpretation $v=\{a\mapsto\tvf\}$. 
We have that for all considered semantics $\sigma$, $v$ is a $\sigma$-interpretation of the SFADF $D=(\{a\}, \{\varphi_{a}=\bot\})$
but there is no SETADF with $v$ being a $\sigma$-interpretation. We thus obtain $\Sigma_{\SETADF}^{\sigma}\subsetneq \Sigma_{SFADF}^{\sigma}$.

\begin{theorem}\label{thm: real SETAF and SFADF}
$\Sigma_{\SETADF}^{\sigma}\subsetneq \Sigma_{SFADF}^{\sigma}$, for $\sigma\in\{\cf, \adm, \stb, \model, \com, \pref, \grd\}$.
\end{theorem}

In the remainder of this section we aim to characterise the difference between $\Sigma_{\SETADF}^{\sigma}$
and $\Sigma_{SFADF}^{\sigma}$. 
To this end we first recall a characterisation of the acceptance conditions of SFADF that can be rewritten as collective attacks.

\begin{lemma}\cite{Wallner19}\label{lemma: SFADF in CNF}
Let $D=(S, L, C)$ be a SFADF. If $s\in S$ has at least one incoming link
then the acceptance condition 
$\varphi_{s}$ can be written in 
CNF containing only negative literals.  
\end{lemma}

It remains to consider those arguments in an SFADF with no incoming links. Such
arguments allow for only two acceptance conditions $\top$ and $\bot$. 
While condition $\top$ is unproblematic (it refers to an initial argument in a 
SETAF), an argument with unsatisfiable acceptance condition cannot be modeled in a SETADF.
In fact,
the different expressiveness of SETADFs and SFADFs is solely rooted in the capability of
SFADFs to set an argument to $\tvf$ via a $\bot$ acceptance condition.

We next give a generic characterisations of the difference between $\Sigma_{\SETADF}^{\sigma}$
and $\Sigma_{SFADF}^{\sigma}$.  

\begin{theorem}\label{thm: delta= sfadf setadf}
For $\sigma\in \{\cf, \adm, \stb, \model, \com, \pref, \grd\}$, 
we have
$\Delta_\sigma = \Sigma_{\SFADF}^{\sigma} \setminus \Sigma_{\SETADF}^{\sigma}$  with
$$
\Delta_\sigma =
\{ \mathbb{V} \in \Sigma_{\SFADF}^{\sigma} \mid \exists v\in \mathbb{V} \text{ s.t. } \forall a: v(a)\in\{ \tvf,\tvu\} \land \exists a: v(a)= \tvf\}.
$$
\end{theorem}
\begin{proof}[Proof sketch]
First for $\mathbb{V} \in \Delta_\sigma$ the interpretation  $v$ cannot be realized in a SETADF as we cannot have 
$v(a)\in \tvf$ without $v(b)\in \tvt$ for some other argument $b$.
On the other hand one can show that when $\mathbb{V}\in\Sigma_{\SFADF}^{\sigma}$ is such that each $v \in \mathbb{V}$ assigns some argument to $\tvt$ one can construct a SETADF $D$ with $\sigma(D)=\mathbb{V}$. This is by the fact that we can rewrite acceptance conditions via Lemma~\ref{lemma: SFADF in CNF} and replace 
$\bot$ acceptance conditions by collective attacks, i.e.\ for each interpretation we add collective
attacks from the arguments set to $\tvt$ to all argument with $\bot$ acceptance condition.
\end{proof}

Next, we provide stronger characterisations of $\Delta_\sigma$ for preferred and stable semantics.

\begin{proposition}\label{prop: V in  K, for prf, stb, mod, V=1}
For $\mathbb{V} \in \Delta_\sigma$ and $\sigma\in\{\stb, \model, \prf\}$  
we have	$|\mathbb{V}|=1$.
For $\sigma\in\{\stb, \model\}$ the unique $v \in \mathbb{V}$ assigns all arguments to $\tvf$.
\end{proposition}
\begin{proof}[Proof sketch]
If a SFADF has a $\sigma$-interpretation $v$ that assigns some arguments to $\tvf$ without assigning an argument to $\tvt$ 
then we have that the arguments assigned to $\tvf$ are exactly the arguments with acceptance condition $\bot$.
For $\stb$ and $\model$ semantics this means all arguments have acceptance condition $\bot$ and the result follows.
Each preferred interpretation assigns arguments with acceptance condition $\bot$ to $\tvf$ and thus
the existence of another preferred interpretation would violate the $\ileq$-maximality of $v$.
\end{proof}
In other words each interpretation-set which is $\sigma$-realizable in SFADFs and contains at least 
two interpretations can be realized in SETADFs, for $\sigma\in\{\stb, \prf, \model\}$.
We close this section with an example illustrating that the above characterisation thus not hold
for $\cf$, $\adm$, and $\comp$.

\begin{example}
Let $D=(\{a, b, c\}, \{\varphi_a= \bot, \varphi_b= \neg c, \varphi_c=\neg b \})$.
We have $\com(D)=\{\{a\mapsto \tvf, b\mapsto \tvu, c\mapsto \tvu\}, \{a\mapsto \tvf, b\mapsto \tvt, c\mapsto \tvf\}, \{a\mapsto \tvf, b\mapsto \tvf, c\mapsto \tvt\}\}$.
By Theorem~\ref{thm: delta= sfadf setadf}, $\com(D)$ cannot be realized as SETADF.
Moreover, as $\com(D) \subseteq \adm(D) \subseteq \cf(D)$ for every ADF $D$, we have
that, despite all three contain more than one interpretation, none of them can be realized via a SETADF.
\end{example}

\section{Discussion}
In this paper, we have characterised the expressiveness 
of SETAFs under 3-valued signatures. 
The more fine-grained notion of 3-valued signatures 
reveals subtle 
differences of the expressiveness of stable and preferred semantics which 
are not present in the 2-valued setting \cite{DvorakFW19}
and enabled us to compare the expressive power
of SETAFs and SFADFs, a subclass of ADFs that allows only for 
attacking links.
In particular, we have exactly characterized the difference 
for conflict-free, admissible, complete, stable, preferred, and grounded 
semantics; this difference is rooted
in the capability of SFADFs to set an initial argument to false.
Together with our exact characterisations on signatures of SETAFs for 
stable, preferred, grounded, and conflict-free semantics, this also yields
the corresponding results for SFADFs. Exact characterisations for admissible
and complete semantics are subject of future work.
Another aspect to be investigated is to which extent our insights 
on labelling-based semantics for SETAFs and SFADFs can help to improve
the performance of reasoning systems.

\paragraph{Acknowledgments}
This research has been supported by
FWF through projects I2854, P30168. 
The second researcher is currently embedded in the Center of Data Science
$\&$ Systems Complexity (DSSC) Doctoral Programme, at the University of Groningen.

\newpage
\appendix
\section{Full Proofs}

\subsection*{Proof of Proposition~\ref{prop:mod=stb}}

We first show the following result.

\begin{lemma}\label{lemma: irrefutable}
Let $D=(S, L, C)$ be an ADF, let $v$ be a model of $D$ and let $s\in S$ be an argument such that all parents of $s$ are attackers. 
Thus,  $\varphi_{s}^v$ is irrefutable if and only if $\varphi_{s}[p/\bot\ : \ v(p)=\tvf]$ is irrefutable.  
\end{lemma}
\begin{proof}
Assume that $D=(S, L, C)$ is an ADF and $v$ is a model of $D$. 
Further, assume $s\in S$ such that $\forall p\in \parents{s}$, $(p, s)$ is an attacking link in $D$. 
Clearly if $\varphi_{s}[p/\bot\ : \ v(p)=\tvf]$ is irrefutable then also $\varphi_{s}^v=
\varphi_s[p/\top : v(p)=\tvt][p/\bot : v(p)=\tvf]$ is irrefutable.
It remains to show that if $\varphi_{s}^v$ is irrefutable then also $\varphi_{s}[p/\bot\ : \ v(p)=\tvf]$ is irrefutable. 
Let $\varphi_{s}'=\varphi_{s}[p/\bot\ : \ v(p)=\tvf]$. 
Towards a contradiction, assume that $\varphi_{s}^v$ is irrefutable and  $\varphi_{s}'$ is not irrefutable.
That is, either $\varphi_{s}'$ is unsatisfiable or it is undecided. 
In both cases, $\varphi_{s}'[p/\top\ :\ v(p)=\tvt]$ is unsatisfiable (as all the links are attacking).
Thus, $\varphi_{s}^v = \varphi_{s}'[p/\top\ :\ v(p)=\tvt]$  is unsatisfiable as well. This is a contradiction with the assumption that $\varphi_{s}^v$ is irrefutable. 
\end{proof}

\begin{proof}[Proof of Proposition~\ref{prop:mod=stb}]
Let $D=(S, L, C)$ be a SFADF. 
Since  $\stb(D)\subseteq\model(D)$ for each ADF $D$, 
it remains to show that each model of $D$ is also a stable model of $D$.  
Towards a contradiction assume that $\model(D)\not \subseteq \stb (D)$. 
Thus, there exists a model $v$ of $D$ which is not a stable model. 
Let $D^v$ be a $\stb$-reduct of $D$ and let $w$ be the unique grounded interpretation of $D^v$. 
Since it is assumed that $v$ is not a stable model,  $v^\tvt\not= w^\tvt$. 
That is, there exists $s\in S$ such that $v(s)=\tvt$ and $w(s)\not=\tvt$. 
Thus, $\varphi_{s}[p/\bot\ : \ v(p)=\tvf] $ is not irrefutable. 
Since, $D$ is a SFADF, all parents of $s$ are attackers. 
Hence,  By Lemma~\ref{lemma: irrefutable}, $\varphi_{s}^v$ is not irrefutable, that is, $v(s)\not=\tvt$. 
This is a contradiction by the assumption that $v(s)=\tvt$. 
Thus, the assumption that $D$ consists of a model which is not a stable model is incorrect.  
\end{proof}

\subsection*{Proof of Theorem \ref{thm:setaf_setadf} }

We first introduce some notation.

\begin{definition}
The function $Lab2Int(\cdot)$ maps three-valued labellings to three-valued interpretations such that
\begin{itemize}
\item 
(a) $Lab2Int(\lambda)(s)=\tvt$ iff $\lambda(s)=\lin$,
\item 
(b) $Lab2Int(\lambda)(s)=\tvf$ iff $\lambda(s)=\lout$, and
\item
(c) $Lab2Int(\lambda)(s)=\tvu$ iff $\lambda(s)=\lundec$.
\end{itemize}	 	  
For a labelling $\lambda$ and an interpretation $I$ we write $\lambda \equiv I$ iff 
$Lab2Int(\lambda) = I$.
For a set $\Lab$ of labellings and a set $\mathbb{V}$ of interpretations we write $\Lab \equiv \mathbb{V}$ iff $\{Lab2Int(\lambda) \mid \lambda \in \Lab \} = \mathbb{V}$.
\end{definition}

With the above notation we can restate Theorem \ref{thm:setaf_setadf} as follows:
For a SETAF $F$ and its associated SETADF $D$ we have
$\sigma_\Lab(F) \equiv \sigma(D)$ for $\sigma \in \{\cf,\adm,\comp,\pref,\grd,\stb\}$.

\begin{proof}[Proof of Theorem \ref{thm:setaf_setadf}]
Let $F=(A, R )$  be a SETAF and $D=(S, L, C)$ be its corresponding SETADF. 
We show that $\{Lab2Int(\lambda) \mid \lambda \in \sigma_\Lab(F)\}=\sigma(D)$. Let $\lambda$ be an arbitrary three-valued labelling and let $v= Lab2Int(\lambda)$. We investigate that $\lambda\in \sigma_\Lab(F)$ if and only if $v\in\sigma(D)$. 

\begin{itemize}
\item Let $\sigma = \adm$.
We first assume that $\lambda\in \adm_\Lab(F)$ and show that $v\in\adm(D)$. 
Consider $s \in S$ and the acceptance condition $\varphi_{s}=\bigwedge_{(B, s)\in R}\bigvee_{a\in B}\neg a$.
If $v(s)=\tvt$ we have that $\lambda(s)=\lin$ and thus that for all $(B, s)\in R$ there exists $b\in B$ s.t. $\lambda(b)=\lout$. 
The latter holds iff for all $(B, s)\in R$ there exists $b\in B$ s.t. $v(b)=\tvf$ iff partial evaluation of $\varphi_{s}$ under $v$ is irrefutable iff $\Gamma_D(v)(s)=\tvt$.
If $v(s)=\tvf$ we have that $\lambda(s)=\lout$ and thus that there exists $ (B, s)\in R$ s.t.\ for all $b\in B$:  $\lambda(b)=\lin$.
The latter holds iff there exists $(B, s)\in R$ s.t.\ for all $b\in B$:  $v(b)=\tvt$ iff $\varphi_{s}^v$ is unsatisfiable iff $\Gamma_D(v)(s)=\tvf$.  
We thus obtain that $v \ileq \Gamma_D(v)$ and therefore $v\in\adm(D)$. 

Now we assume $v\in\adm(D)$ and show that $\lambda\in \adm_\Lab(F)$. 
That is for each $s$ with $\lambda(s)=\lin$  we have $\Gamma_D(v)(s)=\tvt$ and, as argued above,
that for all $(B, s)\in R$ there exists $b\in B$ s.t. $\lambda(b)= out$. 
Moreover for each $s$ with $\lambda(s)=\lout$  we have $\Gamma_D(v)(s)=\tvf$ and, as argued above, that there exists $ (B, s)\in R$ s.t.\ for all $b\in B$:  $\lambda(b)=\lin$.
We obtain $\lambda\in \adm_\Lab(F)$.

\item Let $\sigma\in \{\comp,\pref,\grd\}$.
Let $\lambda\in\comp_\Lab(F)$ and let $\varphi_{s}=\bigwedge_{(B, s)\in R}\bigvee_{a\in B}\neg a$ be the acceptance condition of  $s\in S$ in $D$. 
For complete semantics it is enough to show that $\lambda(s)= \lin$ iff $\Gamma_D(v)(s)=\tvt$ and $\lambda(s)=\lout$ iff $\Gamma_D(v)(s)=\tvf$. 
\begin{itemize}
\item It holds that $\lambda(s)=\lin$ (i.e.\ $v(s)=\tvt$) iff for all $(B, s)\in R$ there exists $b\in B$ s.t. $\lambda(b)=\lout$ iff for all $(B, s)\in R$ there exists $b\in B$ s.t. $v(b)=\tvf$ iff partial evaluation of $\varphi_{s}$ under $v$ is irrefutable iff $\Gamma_D(v)(s)=\tvt$.
\item  On the other hand, $\lambda(s)= \lout$ (i.e.\ $v(s)=\tvf$) iff   there exists $ (B, s)\in R$ s.t.\ for all $b\in B$:  $\lambda(b)=\lin$ iff  there exists $(B, s)\in R$ s.t.\ for all $b\in B$:  $v(b)=\tvt$ iff $\varphi_{s}^v$ is unsatisfiable iff $\Gamma_D(v)(s)=\tvf$.  

\end{itemize}
Now as complete semantics coincide it is easy to verify that also the maximal, i.e.\ the preferred, extensions and the minimal, i.e.\ the grounded, extension coincide. 

\item Let $\sigma = \stb$.
Recall that, by Proposition~\ref{prop:mod=stb}, on SETADFs we have that stable and models semantics coincide.
We will show that $\lambda\in \stb_\Lab(F)$ iff $v\in\model(D)$. 
That is we show that 
for each $s \in S$ we have 
(i) $\lambda(s)= \lin$ iff $v(\varphi_s)=\tvt$ and 
(ii) $\lambda(s)=\lout$ iff $v(\varphi_s)=\tvf$. 
To this end let $\varphi_{s}=\bigwedge_{(B, s)\in R}\bigvee_{a\in B}\neg a$
be the acceptance condition of $s$.
\begin{itemize}
\item It holds that $\lambda(s)=\lin$ (i.e.\ $v(s)=\tvt$) iff for all $(B, s)\in R$ there exists $b\in B$ s.t. $\lambda(b)= out$ iff for all $(B, s)\in R$ there exists $b\in B$ s.t. $v(b)=\tvf$ iff $v(\varphi_s)=\tvt$.
\item  On the other hand, $\lambda(s)= out$ (i.e.\ $v(s)=\tvf$) iff   there exists $ (B, s)\in R$ s.t.\ for all $b\in B$:  $\lambda(b)=\lin$ iff  there exists $(B, s)\in R$ s.t.\ for all $b\in B$:  $v(b)=\tvt$ iff $v(\varphi_s)=\tvf$.  
\end{itemize}

\item Finally let $\sigma = \cf$.
We first assume that $\lambda\in \cf_\Lab(F)$ and show that $v\in\cf(D)$. 
Consider $s \in S$ and the acceptance condition $\varphi_{s}=\bigwedge_{(B, s)\in R}\bigvee_{a\in B}\neg a$.
If $v(s)=\tvt$ we have that $\lambda(s)=\lin$ and thus that for all $(B, s)\in R$ there exists $b\in B$ s.t. $\lambda(b) \not= \lin$.  
The latter holds iff for all $(B,s)\in R$ there exists $b\in B$ s.t. $v(b)\not=\tvt$ iff $\varphi^v_{s}$ is satisfiable.
If $v(s)=\tvf$ we have that $\lambda(s)=\lout$ and thus that there exists $(B, s)\in R$ s.t.\ for all $b\in B$:  $\lambda(b)=\lin$.
The latter holds iff there exists $(B,s)\in R$ s.t.\ for all $b\in B$:  $v(b)=\tvt$ iff $\varphi_{s}^v$ is unsatisfiable.  
We thus obtain that $v\in\cf(D)$. 

Now we assume $v\in\cf(D)$ and show that $\lambda\in \cf_\Lab(F)$. 
That is for each $s$ with $\lambda(s)=\lin$  we have $\varphi^v_{s}$ is satisfiable and, as argued above, that for all  $(B,s)\in R$ there exists $b\in B$ s.t. $\lambda(b) \not= \lin$. 
Moreover for each $s$ with $\lambda(s)=\lout$  we have $\varphi^v_{s}$ is unsatisfiable and, as argued above, that there exists $(B,s)\in R$ s.t.\ for all $b\in B$:  $\lambda(b)=\lin$.
We obtain $\lambda\in \cf_\Lab(F)$.\qedhere

\end{itemize}
\end{proof}

\subsection*{Proof of Proposition \ref{prop:sig_pref}}

We first show that for each SETAF $F$ the set $\pref_\Lab(F)$ satisfies the conditions of the proposition.
The first condition is satisfied as clearly all $\lambda \in \pref_\Lab(F)$ have the same domain.
Now, assume that $\lambda \in \pref_\Lab(F)$ assigns an argument $a$ to $\lout$.
By the definition of conflict-free labellings there is an attack $(B,a)$
such that all arguments $b\in B$ are labeled $\lin$.
Thus Condition (2) is satisfied.	 
For Condition (3), consider $\lambda, \lambda' \in \pref_\Lab(F)$. Notice that there must be a conflict $(S,a)$
with $S \cup \{a\} \subseteq \lambda_\lin \cup \lambda'_\lin$ as otherwise
$(\lambda_\lin \cup \lambda'_\lin, \lambda_\lout \cup \lambda'_\lout,\lambda_\lundec \cap \lambda'_\lundec)$ would be a larger admissible labelling.
If $a \in \lambda'_\lin$ then, by the definition of admissible labellings, there is an attack $(B,b)$ with 
$B \subseteq \lambda'_\lin$ and $b\in S \cap \lambda_\lin$.
Thus $b$ is an argument with $\lambda(b)=\lin$ and $\lambda'(b)=\lout$.
Otherwise if $a \in \lambda_\lin$ then, by the definition of admissible labellings, there is an attack $(B,b)$ with 
$B \subseteq \lambda_\lin$ and $b\in S \cap \lambda'_\lin$.
Then, again by the definition of admissible labellings, there is an attack $(C,c)$ with 
$C \subseteq \lambda'_\lin$ and $c\in B \subseteq \lambda_\lin$.
Thus $c$ is an argument with $\lambda(c)=\lin$ and $\lambda'(c)=\lout$.

Now assume that $\lab$ satisfies all the conditions. We give a SETAF $F_\lab=(A_\lab,R_\lab)$
with $\pref_\Lab(F_\lab) = \lab$. We use
\begin{align*}
A_{\lab} &= \Args{\lab}\\
R_\lab   &= \{(\lambda_{\lin}, a) \mid \lambda \in \lab, \lambda(a)=\lout\}  \cup 
\{ (\lambda_{\lin} \cup \{a\}, a) \mid \lambda \in \lab, \lambda(a)=\lundec\}
\end{align*}

We first show $\pref_\Lab(F_\lab) \supseteq \lab$:
Consider an arbitrary $\lambda \in \lab$: We first show $\lambda \in \cf_\Lab(F_\lab)$.
We first consider $\lout$ labeled arguments. First, if $\lambda(a)=\lout$ for some argument
$a$ then by construction and Condition (2) we have an attack $(\lambda_{\lin}, a)$ 
and thus $a$ is legally labeled $\lout$.
Now towards a contradiction assume there is a conflict $(B,a)$ such that $B \cup \{a\} \subseteq \lambda_{\lin}$.

If $|\lab|=1$, by the construction of $F_\lab$ there is no $(B, a)\in R_\lab$ such that $a \in \lambda_{\lin}$. That is, $a$  is legally labeled $\lin$. 
If $|\lab|>1$, by construction there is a $\lambda' \in \lab$  with $\lambda'_{\lin}=B \setminus \{a\}$, a contradiction to Condition (3).
Thus, $\lambda \in \cf_\Lab(F_\lab)$.
Next we show that $\lambda \in \adm_\Lab(F_\lab)$.
Consider an argument $a$ with $\lambda(a)=\lin$ and an attack $(B,a)$.
Then, by construction there is a $\lambda' \in \lab$  with $\lambda'_{\lin}=B \setminus \{a\}$ and,
by Condition (3), an argument $b \in B$ such that $\lambda(b)=\lout$.
Thus, $\lambda \in \adm_\Lab(F_\lab)$.
Finally we show that $\lambda \in \pref_\Lab(F_\lab)$.
Towards a contradiction assume that there is a $\lambda' \in \adm_\Lab(F_\lab)$
with $\lambda_\lin \subset \lambda'_\lin$.
Let $a$ be an argument such that $\lambda'(a)=\lin$ and $\lambda(a)\in\{\lout,\lundec\}$.
By construction there is either an attack $(\lambda_\lin,a)$ or an attack $(\lambda_\lin \cup \{a\},a)$.
In both cases $\lambda' \not\in \adm_\Lab(F_\lab)$, a contradiction.
Hence, $\lambda \in \pref_\Lab(F_\lab)$.

We complete the proof by showing $\pref_\Lab(F_\lab) \subseteq \lab$:
Consider $\lambda \in \pref_\Lab(F_\lab)$: 
If $\lambda$ maps all arguments to $\lin$ then there is no attack in $R_\lab$ which means that
$\lab$ contains only the labelling $\lambda$.
Thus we can assume that $\lambda(a)=\lout$ for some argument $a$ and there is $(B,a) \in R_\lab$
with  $\lambda(b)=\lin$ for all $b \in B$.
By construction there is $\lambda' \in \lab$ such that $\lambda'_\lin=B$.
Then by construction we have $(B,c) \in R_\lab$ for all $c$ with $\lambda'(c)= \lout$
and $(B \cup \{c\},c) \in R_\lab$ for all $c$ with $\lambda'(c)= \lundec$.
We obtain that $\lambda'_\lin = B = \lambda_\lin$ and thus $\lambda=\lambda'$.	 
\subsection*{Proof of Proposition \ref{prop: cf}}

We first show that for each SETAF $F$ the set $\cf_\Lab(F)$ satisfies the conditions of the proposition.
The first condition is satisfied as clearly all $\lambda \in \cf_\Lab(F)$ have the same domain.
Now, assume that $\lambda \in \cf_\Lab(F)$ assigns an argument $a$ to $\lout$.
By the definition of conflict-free labellings there is an attack $(B,a)$
such that all arguments $b\in B$ are labeled $\lin$.
Thus Condition (2) is satisfied.	 
For Condition (3),  towards a contradiction assume that $(C,\emptyset, \Args{\lab} \setminus C)$
is not conflict-free. Then there is an attack $(B,a)$ such that $B \cup \{a\} \subseteq C$.
But then also $B \cup \{a\} \subseteq \lambda_\lin$ and thus $\lambda \not\in \cf_\Lab(F)$,
a contradiction.
Condition (4) is satisfied as in the definition of conflict-free labellings 
there are no conditions for label an argument $\lundec$. Further, the conditions that allow to label an argument $\lout$ solely depend on the $\lin$ labeled arguments. Since $\lambda_\lout\setminus C\subseteq \lambda_{\lout}$, the condition for arguments labeled $\lout$ is satisfied.   
For Condition (5) consider $\lambda, \lambda' \in \cf_\Lab(F)$ with $\lambda_\lin \subseteq \lambda'_\lin$ 
and $\lambda^*=(\lambda'_\lin,\lambda_\lout \cup \lambda'_\lout, \lambda_\lundec \cap \lambda'_\lundec)$.
First there cannot be an attack $(B,a)$ such that $B \cup \{a\} \subseteq \lambda^*_\lin$
as $\lambda' \in \cf_\Lab(F)$.
Hence, $\lambda'_\lin \cap \lambda_\lout =\emptyset$ and thus $\lambda^*$ is a well-defined labelling.
Moreover, for each $a$ with $\lambda^*(a)=\lout$ there
is an attack $(B,a)$ with $B \subseteq \lambda^*_\lin$ as either $\lambda(a)=\lout$
or $\lambda'(a)=\lout$.
Thus, $\lambda^*\in \cf_\Lab(F)$ and therefore the condition  holds.
For Condition (6) consider $\lambda, \lambda' \in \cf_\Lab(F)$ and
a set $C\subseteq\lambda_{\lout}$ containing an argument $a$ such that $\lambda(a)=\lout$.
That is, there is an attack $(B,a)$ with $B \subseteq \lambda_\lin$
and thus $\lambda_\lin \cup C \not\subseteq \lambda'$.
That is, Condition (6) is satisfied.\smallskip

Now assume that $\lab$ satisfies all the conditions. We give a SETAF $F_\lab=(A_\lab,R_\lab)$
satisfying $\cf_\Lab(F_\lab) = \lab$, where
\begin{align*}
A_{\lab} &= \Args{\lab}\\
R_\lab   &= \{(\lambda_{\lin}, a) \mid \lambda \in \lab, \lambda(a)=\lout\}  \cup 
\{ (B, b) \mid b\in B, \nexists \lambda\in \lab: \lambda_\lin=B \}
\end{align*}
We first show $\cf_\Lab(F_\lab) \supseteq \lab$:
Consider an arbitrary $\lambda \in \lab$: 
First, if $\lambda(a)=\lout$ for some argument
$a$ then by construction and Condition (2) we have an attack $(\lambda_{\lin}, a)$ 
and thus $a$ is legally labeled $\lout$.
Now towards a contradiction assume there is a conflict $(B,a)$ such that $B \cup \{a\} \subseteq \lambda_{\lin}$.
By Condition (3) it cannot be the case that $a \in B$.
Thus, by construction there is a $\lambda' \in \lab$  with $\lambda'_{\lin}=B$,
a contradiction to Condition (6).
Thus, $\lambda \in \cf_\Lab(F_\lab)$.

We complete the proof by showing $\cf_\Lab(F_\lab) \subseteq \lab$:
Consider $\lambda \in \cf_\Lab(F_\lab)$: 
If $\lambda$ maps all arguments to $\lin$ then there is no attack in $R_\lab$ which means that
$\lab$ contains only the labelling $\lambda$.
Thus we can assume that $\lambda(a)\in\{\lout,\lundec\}$ for some argument $a$.		
If $\lambda_\lin \not= \lambda'_\lin$ for all $\lambda' \in \lab$ then by construction of the second part of $R_\lab$ there would be attacks $(\lambda_\lin, b)$ for all $b \in \lambda_\lin$, which 
is in contradiction to $\lambda \in \cf_\Lab(F_\lab)$.
Thus, there is $\lambda' \in \lab$ such that $\lambda'_\lin=\lambda_\lin$.
For arguments $a$ with $\lambda(a)=\lout$ there is an attack $(B,a)$ with $B \subseteq \lambda_\lin$
and, by construction, a $\lambda^* \in \lab$ such that $\lambda^*_{\lin}=B$ and $\lambda^*(a)=\lout$.
By the existence of $\lambda' \in \lab$ and Condition (5) we have that there exists
$\lambda'' \in \lab$ such that $\lambda_\lin=\lambda''_\lin$,
$\lambda'_\lout \subseteq \lambda''_\lout$ and $a \in \lambda''_\lout$.
By iteratively applying this argument for each argument a with $\lambda(a)=\lout$
we obtain that there is a labelling $\hat{\lambda} \in \lab$ such that 
$\lambda_\lin=\hat{\lambda}_\lin$ and $\lambda_\lout \subset \hat{\lambda}_\lout$.
By Condition (4) we obtain that $\lambda \in\lab$.	

\subsection*{Proof of Proposition~\ref{prop:sig_grd}}

We first show that for each SETAF $F$ the set $\grd_\Lab(F)$ satisfies the conditions of the proposition. Towards a contradiction assume that there are  $\lambda, \lambda'\in\grd_\Lab$ with $\lambda\not=\lambda'$. By the definition of grounded labelling $\lambda_\lin$ $\lambda_{\lin}$ are  $\subseteq$-minimal among all complete labellings, thus, $\lambda_{\lin}=\lambda_{\lin}'$. Assume that $\lambda_{\lout}\subset\lambda_{\lout}'$. Since each grounded labelling is conflict-free,  for each $a$ with $a\in\lambda_{\lout}'$ there is $(B, a)$ such that $B\subseteq\lambda_{\lin}'$. Since $\lambda_{\lin}=\lambda_{\lin}'$, $a\in\lambda_{\lout}$. Therefore, $\lambda=\lambda'$. Now, assume that $\lambda\in\grd_\Lab(F)$ assigns an argument $a$ to $\lout$. By the definition of conflict-free labeling there is an attack $(B, a)$ such that $B\subseteq\lambda_{\lin}$. 

Now assume that $\lab$ satisfies all the conditions. We give a SETAF $F_\lab=(A_\lab,R_\lab)$
with $\grd_\Lab(F_\lab) = \lab$. We set
\begin{align*}
A_{\lab} &= \Args{\lab}\\
R_\lab   &= \{(\lambda_{\lin}, a) \mid \lambda \in \lab, \lambda(a)=\lout\}  \cup 
\{ (\lambda_{\lin} \cup \{a\}, a) \mid \lambda \in \lab, \lambda(a)=\lundec\}
\end{align*}
Consider the unique $\lambda\in \lab$ 
and the unique $\lambda^G \in \grd_\Lab(F_\lab)$.
For each argument $a \in \lambda_\lin$ we have that $a$ is not attacked in  $F_\lab$
and thus $a \in \lambda^G_\lin$.
For each argument $a \in \lambda_\lout$ there is an attack  $(\lambda_{\lin}, a)$ in $F_\lab$
and as $\lambda_\lin \subseteq \lambda^G_\lin$ by the definition of complete labellings 
we have $a \in \lambda^G_\lout$.
Finally for each argument $a \in \lambda_\lundec$ the attack $(\lambda_{\lin} \cup \{a\}, a)$ is the only attack towards $a$ in $F_\lab$. 
Thus, by the definition of complete labellings, we have that $a$ is neither labelled $\lin$ nor $\lout$ in $F_\lab$ and therefore $a \in \lambda^G_\lundec$.
We obtain that $\lambda^G = \lambda$ and thus $\grd_\Lab(F_\lab) = \lab$.

\subsection*{Proof of Theorem~\ref{thm: real SETAF and SFADF}}

$\Sigma_{\SETADF}^{\sigma}\subseteq \Sigma_{SFADF}^{\sigma}$ follows from Lemma~\ref{lem: SETADF is a SFADF}.
For showing $\Sigma_{\SETADF}^\adm\subsetneq\Sigma_{\SFADF}^\adm$, let $\mathbb{V}=\{\{a\mapsto\tvu, b\mapsto\tvu\}, \{a\mapsto\tvu, b\mapsto\tvf\}, \{a\mapsto\tvt, b\mapsto\tvf\}\}$ be an interpretation-set. A witness of $\adm$-realizability of $\mathbb{V}$  in SFADFs is $D=(\{a, b\}, \{\varphi_{a}=\neg a\lor\neg b, \varphi_b=\bot\})$. However, $\mathbb{V}$ is not realizable by any SETADF for admissible interpretations (cf.~Proposition~\ref{prop:sig_adm}).
To show $\Sigma_{\SFADF}^\sigma\not\subseteq \Sigma_{\SETADF}^\sigma $, for $\sigma\in\{ \stb, \model, \com, \pref, \grd\}$, let $\mathbb{V}=\{\{a\mapsto\tvf\}\}$. The interpretation  $\mathbb{V}$ is $\sigma$-realizable in SFADFs for $\sigma\in\{ \stb, \model, \com, \pref, \grd\}$, and a witness of $\sigma$-realizability of $\mathbb{V}$ in SFADFs is $D=(\{a\}, \{\varphi_{a}=\bot \})$. However, $\mathbb{V}$ cannot be realized by any SETADF for semantics
$\sigma\in\{ \adm, \stb, 
\pref, \grd\}$
(cf. Propositions~\ref{prop:sig_stb}--\ref{prop:sig_grd}). 
The result for $\sigma=\mod$ follows from Proposition~\ref{prop:mod=stb} and
for $\sigma=\comp$ by $|\mathbb{V}|=1$ (i.e.\ complete and grounded semantics have to coincide). Further,  $\cf(D)$ is not $\cf$-realizable with any SETADF. 

\begin{lemma}\label{lemma: K is not adm-rea in SETAF}
Given an   interpretation-set  $\mathbb{V}\in \Delta_\sigma$,
for $\sigma\in\{\adm, \stb, \model,$ $ \com, \pref, \grd\}$. 
Let $v\in \mathbb{V}$ be a non-trivial interpretation in which $v(a)=\tvf/\tvu$, for each argument $a$. In all SFADFs that realize $\mathbb{V}$ under $\sigma$, the acceptance conditions of all arguments assigned to $\tvf$ by $v$ are equal to $\bot$. 
\end{lemma}
\begin{proof}
Let $D$ be a SFADF that realizes $\mathbb{V}$ under $\sigma$, for $\sigma\in\{\adm, \stb, \model, \com,$ $ \pref, \grd\}$. Let $v\in \mathbb{V}$ be an non-trivial interpretation that assigns all arguments either to $\tvf$ or $\tvu$.  Towards a contradiction,   assume that there exists  an argument $a$ which is assigned to $\tvf$ by $v$, and $\varphi_{a}\not=\bot$ in $D$. 
First we show that $\mathbb{V}$ cannot be $\adm$-realizable in SFADFs.  Since $a$ is assigned to $\tvf$ in $v$ the acceptance condition of $a$ cannot be equal to $\top$. By Lemma~\ref{lemma: SFADF in CNF}, the acceptance condition of $a$ is in CNF and having only negative literals. Since all $b\in \parents{a}$ are either assigned to $\tvf$ or $\tvu$ by $v$, $\varphi_{a}^v$ cannot be unsatisfiable. That is, $v(a)\not\leq_i\Gamma_D(v)(a)$. Therefore, $v$ is not an admissible interpretation of $D$. Thus, any $\mathbb{V}$ that contains $v$ is not $\adm$-realizable in SFADF. To complete the proof it remains to see that for each of the remaining semantics, 
each $\sigma$-interpretation is also admissible.
\end{proof}

\subsection*{Proof of Theorem \ref{thm: delta= sfadf setadf}}
To show that 
$ \Delta_\sigma =
\{ \mathbb{V} \in \Sigma_{\SFADF}^{\sigma} \mid \exists v\in \mathbb{V} \text{ s.t. } \forall a: v(a)\in\{ \tvf,\tvu\} \land \exists a: v(a)= \tvf\}$, let $\mathbb{V}$ be an arbitrary interpretation-set of $\Delta_\sigma$. 
By the definition of $\Delta_\sigma$, $\mathbb{V}\in \Sigma_{\SFADF}^{\sigma}$ and $\mathbb{V}\not\in \Sigma_{\SETADF}^{\sigma}$. It remains to show that there exists $v\in \mathbb{V}$ that assigns at least an argument to $\tvf$ but none of the arguments to $\tvt$.
Towards a contradiction, assume that there exists no such interpretation
and let $D=(S, L, C$ be an arbitrary $\SFADF$ with $\sigma(\SFADF)=\mathbb{V}$.
Notice that by Lemma~\ref{lemma: SFADF in CNF} all acceptance conditions of $D$ that are not equal to $\bot$
can be transformed to be in SETADF form. 
Thus we can focus on the arguments with acceptance condition $\bot$.
As, under the above assumption, each $v\in \mathbb{V}$ that assigns an argument to $\tvf$
also assigns an argument $b$ to $\tvt$ it is easy to verify that we can replace 
$\bot$ acceptance conditions by $\bigwedge_{s\in S} \neg s$ without changing the semantics.
That is, we can transform $D$ to an equivalent SETADF and thus 
$\mathbb{V} \in \Sigma_{\SETADF}$. This is a contradiction by the definition of $\Delta_\sigma$ and 
we obtain that there exists $v\in \mathbb{V}$ that assigns all arguments to either $\tvf$ or $\tvu$. 

On the other hand, let $\mathbb{V}$ be an interpretation-set that is $\sigma$-realizable in SFADF such that there exists $v\in \mathbb{V}$ that assigns at least one argument to $\tvf$ and none of the arguments to $\tvt$. We show that $\mathbb{V}\not\in\Sigma_{SETADF}^{\sigma}$. By Lemma~\ref{lemma: K is not adm-rea in SETAF}, 
in any $\SFADF$  with $\sigma(\SFADF)=\mathbb{V}$ the acceptance conditions of all arguments assigned to $\tvf$ by $v$ are equal to $\bot$. Therefore, $D$ is not $\sigma$-realizable in any SETADF.
That is, $\mathbb{V}\in \Delta_\sigma$. 

\subsection*{Proof of Proposition~\ref{prop: V in  K, for prf, stb, mod, V=1}}

Consider $\mathbb{V}\in \Delta_\sigma$, for $\sigma\in\{\stb, \model, \prf\}$ and 
let $v\in \mathbb{V}$ be an interpretation that assigns all arguments to either $\tvf$ or $\tvu$ (since $\mathbb{V}\in \Delta_\sigma$, such a $v$ exists). By Lemma~\ref{lemma: K is not adm-rea in SETAF}, the acceptance condition of all arguments that are assigned to $\tvf$ by $v$ is equal to $\bot$ in all SFADFs that realize $\mathbb{V}$ under $\sigma\in\{\stb, \model, \prf\}$. Let $D=(S, L, C)$ be a witness  of  $\sigma$-realizibility of $\mathbb{V}$ in SFADFs, under $\sigma\in\{\stb, \model, \prf\}$. 

First, if all arguments are assigned to $\tvf$ in $v$, the acceptance conditions of all arguments are $\bot$ in  SFADF $D$ and $|\sigma(D)|= 1$. 
Now assume that  $v$ assigns some  arguments to $\tvu$. Thus, $V$ cannot be $\model$ or $\stb$-realized in any ADF.   It remains to consider $\pref$ semantics. 
Let $B = \{s \in S \mid v(b)=\tvu\}$. For each $s\in S\setminus B$, by Lemma~\ref{lemma: K is not adm-rea in SETAF},  $\varphi_s=\bot$ in $D$. Therefore, in all $v'\in \mathbb{V}$,  $v'(s)=\tvf$ for $s\in S\setminus B$.  For each $v'\not =v$ in $\mathbb{V}$  there exists  at least $b\in B$ such that $v'(b)\not=\tvu$, therefore, $v< v'$.
By the definition of preferred interpretations $v$ cannot be a preferred interpretation.
Thus, $|\pref(D)|= 1$ and therefore, the assumption $|\mathbb{V}|=1$.
Summarizing the two cases we have that interpretation set $\mathbb{V}\in \Delta_\sigma$,  
for $\sigma\in\{\stb, \model, \prf\}$ consist of only one interpretation.

\end{document}